\documentclass[11pt]{article}

\usepackage[margin=1in]{geometry}

\usepackage{amsmath,amssymb,amsthm}
\usepackage{graphicx}
\usepackage{booktabs}
\usepackage{microtype}
\usepackage{indentfirst}
\usepackage{csquotes}
\usepackage{xcolor}
\usepackage{paralist}
\usepackage{hyperref}
\usepackage[numbers]{natbib}
\usepackage{thmtools}
\usepackage{thm-restate}
\usepackage{longtable}
\usepackage{subcaption}

\hypersetup{
  colorlinks=true,
  linkcolor=blue,
  citecolor=blue,
  urlcolor=blue,
  filecolor=blue
}

\newtheorem{theorem}{Theorem}

\newtheorem{corollary}[theorem]{Corollary}

\newtheorem{assumption}[theorem]{Assumption}

\newcommand{\calU}{\mathcal{U}}
\newcommand{\calC}{\mathcal{C}}

\title{Retrieve, Then Classify: Corpus-Grounded Automation of Clinical Value Set Authoring}

\author{Sumit Mukherjee, Juan Shu, Nairwita Mazumder, \\
Tate Kernell, Celena Wheeler, Shannon Hastings, Chris Gibbons \\
\textit{Oracle Health Data Intelligence}}
\date{\today}

\begin{document}

\maketitle


\begin{abstract}
Clinical value set authoring --- the task of identifying all codes in a standardized vocabulary that define a clinical concept --- is a recurring bottleneck in clinical quality measurement and phenotyping. A natural approach is to prompt a large language model (LLM) to generate the
required codes directly, but structured clinical vocabularies are large, version-controlled, and not reliably memorized during pretraining. We propose Retrieval-Augmented Set Completion (RASC): retrieve the $K$ most similar existing value sets from a curated corpus to form a candidate
pool, then apply a classifier to each candidate code. Theoretically, retrieve-and-select can reduce statistical complexity by shrinking the effective output space from the full vocabulary to a much smaller retrieved candidate pool. We demonstrate the utility of RASC on 11,803 publicly available VSAC value sets, constructing the first large-scale benchmark for this task. A cross-encoder fine-tuned on SAPBert achieves AUROC~0.852 and value-set-level F1~0.298, outperforming a simpler three-layer Multilayer Perceptron (AUROC~0.799, F1~0.250) and both reduce the number of irrelevant candidates per true positive from 12.3 (retrieval-only) to approximately 3.2 and 4.4 respectively. Zero-shot GPT-4o achieves value-set-level F1~0.105, with 48.6\% of returned codes absent from VSAC entirely. This performance gap widens with increasing value set size, consistent with RASC's theoretical advantage. We observe similar performance gains across two other classifier model types, namely a cross-encoder initialized from pre-trained SAPBert and a LightGBM model, demonstrating that RASC's benefits extend beyond a single model class. The code to download and create the benchmark dataset, as well as the model training code is available at:  \href{https://github.com/mukhes3/RASC}{https://github.com/mukhes3/RASC}. 
\end{abstract}

\section{Introduction}
\label{sec:intro}

Clinical value set authoring is the task of identifying, from a controlled
medical vocabulary, all codes that define a given clinical concept ---
for example, every SNOMED-CT and ICD-10-CM code that constitutes
``Type~2 Diabetes Mellitus'' for use in a quality measure.
Value sets are a foundational infrastructure of clinical quality measurement
and phenotyping~\cite{vsac2013}, yet authoring them remains largely manual:
a clinical informaticist must search a vocabulary of $10^5$--$10^6$ codes,
draw on domain expertise to judge relevance, and produce a list that is both complete and precise. The Value Set Authority Center (VSAC) now hosts over $10,000$ such sets, with creation rates accelerating, making the scalability of manual authoring an
increasingly practical concern.

The same structure --- construct a target subset $S^* \subseteq \mathcal{U}$
from a large discrete universe $\mathcal{U}$ given a concept query --- appears in gene panel construction~\cite{panelapp} and systematic review inclusion~\cite{o2019using}, where the universe is a gene catalogue or a paper database respectively.
What these tasks share is that $\mathcal{U}$ is large, structured, and version-controlled; the target $S^*$ is small relative to $\mathcal{U}$; a corpus of related expert-curated sets already exists; and errors carry real
downstream consequences.

A natural approach is to prompt a large language model (LLM) to generate $S^*$ directly from the concept name.
This is appealing in its simplicity, but clinical code identifiers are not natural language --- they are structured, version-controlled strings that are not reliably memorized from pretraining data. We find empirically that GPT-4o (chosen due to it's widespread use) returns codes absent from VSAC entirely at a 48.6\% rate, consistent with prior work on LLM reliability in clinical NLP~\cite{sivarajkumar2024gpt4}. There is also a more domain specific limitation: prompting an LLM from scratch ignores the fact that most domains already possess a corpus of prior expert curation that encodes structured knowledge about $\mathcal{U}$. For instance, when a clinical informaticist authors a new value set for ``Type~2 Diabetes
Without Complications,'' the codes they need are largely a near-subset of codes already assembled in existing value sets for related diabetic concepts.

To overcome these, we propose \emph{retrieval-augmented set completion} (RASC), a two-stage framework that exploits this signal directly. In Stage~1, semantic similarity search retrieves the $K$ most similar existing sets from the corpus, forming a candidate pool $\mathcal{C} \subseteq \mathcal{U}$ that is small relative to $\mathcal{U}$ but enriched for members of $S^*$. In Stage~2, a lightweight binary classifier decides, for each item in $\mathcal{C}$, whether it belongs to $S^*$. This reduces an intractable generation problem over $|\mathcal{U}|$ items to a tractable classification problem over a pool of a few hundred candidates.


This paper makes three contributions.
First, we formalize the set completion problem and the RASC framework, and provide a learning-theoretic analysis showing that RASC's sample complexity scales with $\log K$ rather than $\log N$, and derive an explicit sample complexity gap. Second, we construct the first large-scale benchmark for clinical value set authoring and demonstrate that a cross-encoder substantially outperforms both the retrieval-only baseline and zero-shot LLM generation across all value set types and sizes. Finally, we isolate the contribution of retrieval grounding from learned classification via a pool-grounded LLM generator.

\section{Related Work}

\subsection{Retrieval-augmented generation.}
Retrieval-augmented generation (RAG)~\cite{lewis2020retrieval} conditions a
generative model on retrieved documents to ground outputs in evidence. RASC
differs in a crucial way: retrieval in RAG provides \emph{context} for
generation, while retrieval in RASC provides the \emph{feasible set} itself.
The classifier is not conditioned on retrieved text; it operates over retrieved
\emph{items}, and the retrieved pool hard-constrains the output space. This
distinction has non-trivial theoretical consequences: in RAG, the generative
model still samples from $\calU$; in RASC, the output is guaranteed to be
a subset of $\calC$.

\subsection{Set prediction and subset selection.}
Set prediction~\cite{zhang2020deep} and multi-label
classification~\cite{tsoumakas2007multi} are closely related problems. Our
formulation is distinct in that the label space is not fixed across examples:
each query $q$ induces a different candidate pool $\calC(q)$, and the
classifier must generalize across pools. This makes our setting closer to
extreme multi-label classification~\cite{bhatia2015sparse} over dynamic
label sets, with the additional structure that negatives are structured
(items drawn from proximal sets) rather than randomly sampled.

\subsection{Automated clinical coding and value sets.}
Automated ICD coding from clinical
documents~\cite{mullenbach2018explainable,huang2022plm} is superficially
similar but fundamentally different: it assigns codes to documents, not
constructs definitional sets of codes. Value set quality and
maintenance~\cite{steele2017quality, mo2015desiderata} and clinical
phenotyping~\cite{kirby2016phekb, yu2018phenorm} are related downstream
applications. To our knowledge no prior work has framed value set construction
as a learning problem or constructed a benchmark dataset for it.

\begin{figure}[!h]
    \centering
    \includegraphics[width=1\linewidth]{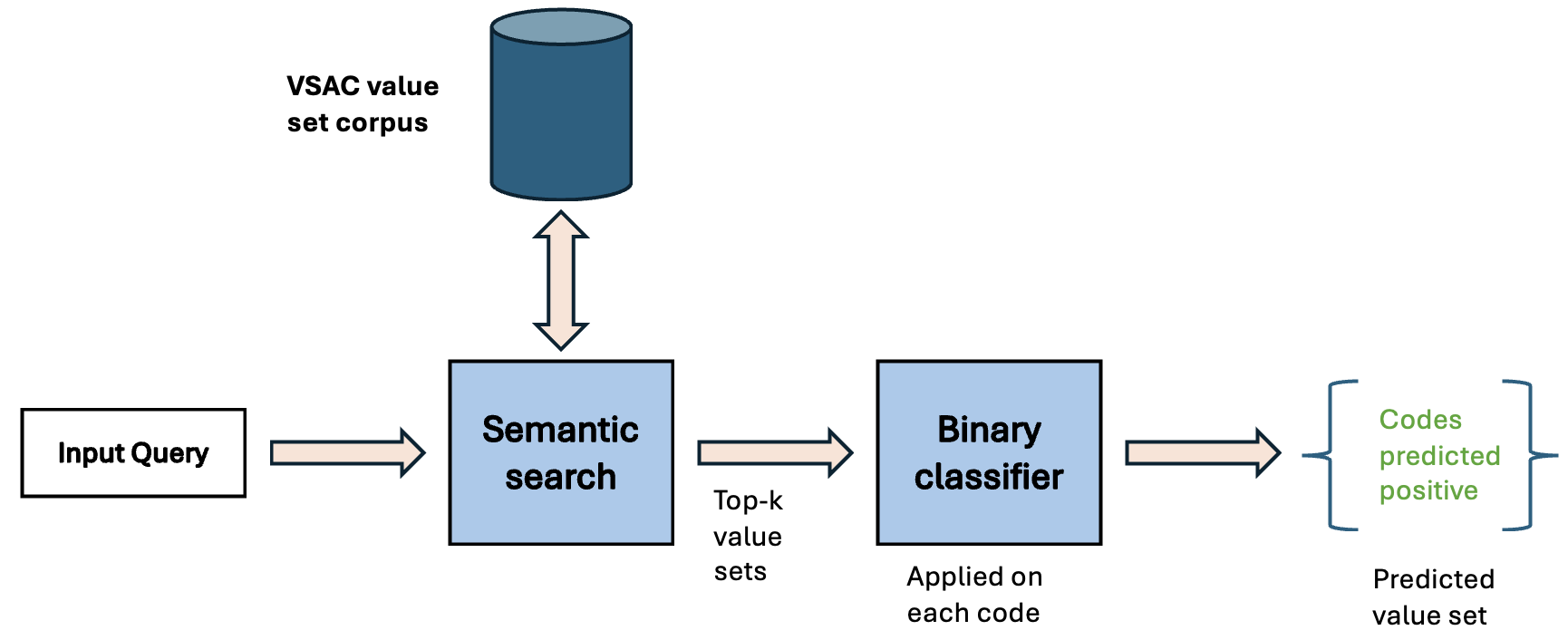}
    \caption{Overview of RASC workflow.}
    \label{fig:rasc}
\end{figure}

\section{Statistical Motivation for RASC}
\label{sec:theory}

Let $\mathcal{U}$ denote the universe of $(\text{code}, \text{system})$ pairs
with $|\mathcal{U}| = N$.
For a query $q$ (a value set title), the task is to identify the target label
set $Y(q) \subseteq \mathcal{U}$, where $|Y(q)| \ll N$ since each value set
contains only a small fraction of the full ontology (see Figure~\ref{fig:rasc} for visual overview).
We assume a latent relevance score $\theta_q : \mathcal{U} \to \mathbb{R}$
and threshold $\tau$ such that $c \in Y(q)$ if and only if
$\theta_q(c) \ge \tau$, with a margin $\gamma > 0$ separating members from
non-members.
A learned estimator $\hat{\theta}_n$ approximates $\theta_q$ from $n$
training examples, with sub-Gaussian score errors that concentrate uniformly
over any finite candidate set.

The primary comparator of RASC is direct generation of value sets via LLMs. To compare them mathematically, we specify two predictors.
\emph{Direct prediction} $\hat{Y}_{\mathrm{dir}}(q)$ thresholds
$\hat{\theta}_n$ over all of $\mathcal{U}$.
\emph{Retrieval-restricted prediction} $\hat{Y}_{\mathrm{RASC}}(q)$ first
retrieves a pool $\mathcal{C}(q)$ with $|\mathcal{C}(q)| \le K \ll N$ and
thresholds only within that pool.
Exact recovery is impossible when $Y(q) \nsubseteq \mathcal{C}(q)$, which
occurs with probability $\varepsilon_{\mathrm{ret}}$.

\begin{corollary}[Sample complexity gap; proof in Appendix~\ref{sec:appendix_theory}]
\label{cor:sample_complexity}
Fix $\delta \in (0,1)$.
Direct prediction achieves $\Pr(\hat{Y}_{\mathrm{dir}}(q) \neq Y(q)) \le \delta$
whenever
\[
    n \;\ge\; \frac{2\sigma^2}{\gamma^2} \log\frac{2N}{\delta}.
\]
Retrieval-restricted prediction achieves
$\Pr(\hat{Y}_{\mathrm{RASC}}(q) \neq Y(q)) \le \delta$
whenever $\delta > \varepsilon_{\mathrm{ret}}$ and
\[
    n \;\ge\; \frac{2\sigma^2}{\gamma^2}
    \log\frac{2K}{\delta - \varepsilon_{\mathrm{ret}}}.
\]
\end{corollary}

When $\varepsilon_{\mathrm{ret}}$ is small and $K \ll N$, RASC requires only
$O((\sigma^2/\gamma^2)\log K)$ samples to reach a target recovery probability,
versus $O((\sigma^2/\gamma^2)\log N)$ for direct prediction.
In clinical vocabulary settings $N/K$ is expected to be very large, so this
reduction in log-complexity is substantial.
The $\varepsilon_{\mathrm{ret}}$ term sets a hard floor on achievable error
that no classifier can overcome regardless of $n$; when retrieval recall is
high this floor is low, and the estimation advantage of RASC dominates. It is worth noting here that the comparison presented here is not specific to auto-regressive language models, but rather to a simplified direct generator model for the sake of tractability.

\section{Methods}
\label{sec:methods}

\subsection{VSAC Corpus}
\label{sec:methods:dataset}
We bulk-downloaded all publicly available value sets from the Value Set
Authority Center (VSAC)~\cite{vsac2013} via the HL7 FHIR
\texttt{ValueSet/\$expand} API under a free UMLS license, paginating through
the \texttt{ValueSet} search endpoint to enumerate all OIDs.
The resulting corpus, after filtering out value-sets containing less than three codes, contains \textbf{11,803 value sets} spanning 15
terminology systems and 847 publisher organizations.
Value set sizes follow a heavy-tailed distribution (median~9 codes;
95th percentile~312 codes); 80.4\% carry no human-authored description;
over 80\% draw codes from a single system.

\subsection{Semantic Retrieval}
\label{sec:methods:retrieval}

We embed value set titles using SAPBert~\cite{liu2021sapbert}, a BERT encoder
pretrained on UMLS synonym pairs that maps biomedical entity mentions to a
metric space robust to surface-form variation.
We use CLS-token pooling throughout, instantiating the model explicitly with
\texttt{pooling\_mode\_cls\_token=True} to ensure consistent behavior across
index construction and query-time inference.
Robustness to typographic variation is validated empirically: case-perturbed
title queries achieve 98.5\% top-1 accuracy on a held-out sample of 200
value sets.

All 11,803 title embeddings are stored in a FAISS \texttt{IndexFlatIP}
index~\cite{johnson2019faiss}; at this scale exact search completes in under
5\,ms per query.
We embed titles only, as 80.4\% of value sets lack descriptions, making
description-augmented embeddings inconsistent at inference time.

For a target value set $v^*$ with title $t^*$, we retrieve the $K = 10$ most
similar value sets by cosine similarity (excluding $v^*$ itself) and take
their union of codes as the candidate pool $\mathcal{C}(v^*)$, collapsing
duplicate $(\text{code}, \text{system})$ pairs by retaining the
highest-scoring entry.
Each candidate $c$ inherits a similarity score $s_c \in [-1,1]$ from its
source value set.
Retrieval coverage is quantified by
\begin{equation}
    \mathrm{RR}@K(v^*)
    = \frac{|\,\mathrm{codes}(v^*) \cap \mathcal{C}(v^*)\,|}
           {|\,\mathrm{codes}(v^*)\,|},
    \label{eq:rrk}
\end{equation}
which upper-bounds downstream classifier recall.

\subsection{Train/Validation/Test Splits}
\label{sec:methods:splits}

The splitting unit is the value set: all candidate pairs from a given value
set are assigned to the same split, preventing label leakage across concepts.
We apply a two-level stratification scheme.
First, the two largest publishers---Clinical Architecture ($n=596$) and CSTE
Steward ($n=590$)---are held out entirely for test, evaluating
cross-organization generalization to stewards unseen during training.
The remaining value sets are divided 70/15/15 (train/val/test) by stratified
sampling on $(\texttt{type},\, \texttt{publisher\_bin})$, preserving the
joint type--publisher distribution across splits.
Value sets with fewer than three member codes or missing expansions are
excluded prior to split assignment.
Table~\ref{tab:splits} reports pair-level counts after candidate-pool
construction.

\subsection{Feature Representation}
\label{sec:methods:features}

Each example is a (value set, candidate code) pair.
The feature vector $\mathbf{x} \in \mathbb{R}^{1545}$ concatenates: (1)
the SAPBert CLS embedding of the value set title
$\mathbf{e}_{t^*} \in \mathbb{R}^{768}$; (2) the SAPBert CLS embedding of
the candidate display name $\mathbf{e}_{d_c} \in \mathbb{R}^{768}$; (3) a
one-hot code system indicator over eight systems (SNOMED-CT, ICD-10-CM,
RxNorm, LOINC, CPT, ICD-10-PCS, HCPCS, OTHER); and (4) the retrieval
similarity score $s_c$.
The title embedding is identical for all candidates evaluated against the
same target; the code system indicator encodes structural constraints absent
from the semantic embeddings (e.g.\ a SNOMED-CT value set should not contain
RxNorm codes).
Embeddings are computed once over all unique strings prior to feature
assembly---7,052 unique titles and 202,573 unique display names---and cached,
reducing encoder calls proportionally to the deduplication factor.

\subsection{Classification Models}
\label{sec:methods:models}

Code inclusion is posed as binary classification: given $\mathbf{x}$,
estimate $\hat{y} = P(c \in \mathrm{codes}(v^*) \mid \mathbf{x})$.
Class imbalance is
addressed via weighted binary cross-entropy,
decision threshold tuned to maximize F1 on the validation set.

\paragraph{LightGBM.}
A gradient-boosted tree classifier trained on the 1545-dimensional feature
vector, serving as a computationally efficient non-neural reference.

\paragraph{MLP.}
A three-hidden-layer network (1545$\to$512$\to$256$\to$64$\to$1) with batch
normalization, ReLU activations, and dropout ($p=0.3$) after each layer
(941k parameters).
Training uses AdamW ($\text{lr}=3{\times}10^{-4}$, weight decay $10^{-4}$)
with \texttt{ReduceLROnPlateau} scheduling, batch size 2048, and early
stopping on validation loss (patience~5).

\paragraph{Cross-encoder.}
We fine-tune SAPBert (110M parameters) as a cross-encoder by presenting the
value set title and candidate display name as a single sentence pair:
\begin{equation}
    \texttt{[CLS]}\ t_1 \cdots t_m\ \texttt{[SEP]}\ d_1 \cdots d_n\
    \texttt{[SEP]},
    \label{eq:ce_input}
\end{equation}
truncated to 128 tokens.
The \texttt{[CLS]} representation is passed to a linear classification head
$\hat{y} = \sigma(\mathbf{w}^\top \mathbf{z} + b)$.
Unlike the MLP, which interacts title and code only through concatenation of
independently computed embeddings, the cross-encoder attends jointly across
all title and code tokens through 12 transformer layers, enabling
fine-grained token-level interactions that the bi-encoder feature space
cannot represent.
Training uses AdamW ($\text{lr}=2{\times}10^{-5}$, weight decay $10^{-2}$)
with linear warmup over 6\% of steps, gradient accumulation over 4 steps
(effective batch size~64), fp16 mixed precision, and gradient checkpointing. The cross-encoder was trained for 3 epochs.

\subsection{Evaluation}
\label{sec:methods:eval}

\paragraph{Baselines.}
The \emph{retrieval-only} baseline predicts inclusion for every candidate,
achieving recall equal to $\mathrm{RR}@K$ at the value-set level and
precision equal to the pool positive rate.
The \emph{LLM baseline} evaluates GPT-4o zero-shot on
the full test set: each value set is presented with only its name and allowed
code systems, with no access to the candidate pool or VSAC corpus.
Outputs are matched via exact $(\text{code}, \text{system})$ pair matching
after normalizing system URIs to canonical short forms.

\paragraph{Metrics.}
All methods are evaluated at the \emph{value-set level}: precision, recall,
and F1 are computed per value set and macro-averaged across the test set,
providing a unified frame for classifiers and GPT-4o alike. For classifiers we additionally report code-pair-level AUROC and average
precision (AP), which are threshold-free.
For the LLM baseline we report a \emph{hallucination rate}: the fraction of
returned codes absent from the full test corpus, indicating fabricated
identifiers rather than merely incorrect code assignments.
Post-hoc stratified analyses are conducted across value set type, size, and
publisher; precision is reported separately on the $\mathrm{RR}@K=1.0$ subset
, where retrieval is perfect and any precision deficit is
attributable to the model.

\subsection{Comparing LLM-as-Classifier with Zero-Shot Generation}
\label{sec:methods:llm_comparison}

To isolate the contribution of retrieval grounding from the contribution of
learned classification, we compare two GPT-4o conditions on a
cost-constrained subsample of the test set.
The zero-shot generation condition is described in Section~\ref{sec:methods:eval};
here we introduce a complementary \emph{LLM-as-classifier} condition in which
the retrieved candidate pool is provided directly in context.

We restrict this comparison to value sets with $|S^*| \le 50$ true codes,
for which presenting the full candidate pool in-context is tractable, and
draw a stratified random sample of 100 value sets across clinical types.
For each value set, the model receives the value set title, the allowed code
systems, and the full candidate pool as a list of \texttt{(code, system,
display)} triples, and is asked to return the subset it judges relevant.
The system prompt is otherwise identical to the zero-shot generation prompt
(Appendix~\ref{sec:appendix:prompt}), with the addition of an instruction to
select only from the provided candidates.

This condition is not practical at scale: candidate pools of 100--600 codes
produce prompts of 1{,}000--6{,}000 tokens per value set, making inference
roughly two orders of magnitude more expensive than the MLP and fundamentally
limited to small value sets by context length constraints.

\section{Results}

\subsection{Dataset Statistics}

\begin{table}[t]
\centering
\caption{
    Pair-level split statistics after candidate-pool construction.
    The lower positive rate in the test set reflects the two held-out
    publisher organizations (Clinical Architecture and CSTE Steward), whose
    value sets tend to be smaller and more specialized, yielding lower
    retrieval recall and a smaller fraction of true codes in the candidate
    pool.
}
\label{tab:splits}
\begin{tabular}{lrrrr}
\toprule
Split & Examples & Positives & Pos.\ Rate & Publisher source \\
\midrule
Train      & 3,520,599 & 432,851 & 12.3\% & In-distribution \\
Validation &   802,303 & 103,657 & 12.9\% & In-distribution \\
Test       & 2,011,009 & 151,630 &  7.5\% & In-dist.\ + held-out \\
\bottomrule
\end{tabular}
\end{table}

Table~\ref{tab:splits} reports pair-level example counts after candidate-pool
construction across all value sets.
The positive rate is 7.5\% on the test set versus approximately 12\% on
train and validation, reflecting the held-out publishers: Clinical Architecture
and CSTE Steward author smaller, more specialized value sets with lower
retrieval overlap against the in-distribution corpus.
The overall ratio of roughly 1 positive per 13 candidates motivates the
classification stage: reducing this review burden is the central objective.

\subsection{Classifier Results}

\begin{table}[t]
\centering
\caption{
    Code-pair-level test set performance.
    Decision thresholds are tuned on the validation set.
}
\label{tab:results}
\begin{tabular}{lccccc}
\toprule
Model & AUROC & Avg.\ Prec. & Precision & Recall & F1 \\
\midrule
Retrieval-only & ---   & ---   & 0.092 & 1.000 & 0.140 \\
LightGBM       & 0.745 & 0.238 & 0.173 & 0.316 & 0.190 \\
MLP            & 0.799 & 0.271 & 0.226 & 0.470 & 0.250 \\
Cross-Encoder  & 0.852 & 0.374 & 0.272 & 0.483 & 0.298 \\
\bottomrule
\end{tabular}
\end{table}

Table~\ref{tab:results} reports code-pair-level test performance.
All three classifiers substantially outperform the retrieval-only baseline on
precision and F1, and performance improves monotonically from LightGBM to MLP
to Cross-Encoder across all threshold-free metrics.

\paragraph{Precision--recall tradeoff.}
The retrieval-only baseline presents approximately 13 irrelevant candidates
per true positive; LightGBM reduces this ratio while retaining 31.6\% of
true codes, and the MLP improves further to 47.0\% recall at higher
precision.
The Cross-Encoder achieves the best performance on all metrics (AUROC~0.852,
AP~0.374, F1~0.298), representing gains of +0.053 AUROC and +0.103 AP over
the MLP.
The precision gain (+0.046) is substantially larger than the recall difference
(+0.013), indicating that the cross-encoder's joint attention over
(title, code) token pairs primarily reduces false positives rather than
recovering additional true codes missed by the MLP.

\subsection{Value-Set Level Results}

\begin{table}[t]
\centering
\caption{
    Value-set-level test performance (macro averages).}
\label{tab:vs_results}
\begin{tabular}{lcccc}
\toprule
Model & Precision & Recall & F1 & SE(F1) \\
\midrule
Retrieval-only        & 0.092 & 0.553 & 0.136 & 0.003 \\
LightGBM              & 0.173 & 0.316 & 0.190 & 0.005 \\
MLP                   & 0.226 & 0.470 & 0.250 & 0.005 \\
Cross-Encoder         & 0.272 & 0.483 & 0.298 & 0.006 \\
GPT-4o (zero-shot)    & 0.169 & 0.100 & 0.105 & 0.004 \\
\bottomrule
\end{tabular}
\end{table}

Table~\ref{tab:vs_results} reports macro-averaged value-set-level results,
placing all methods on a common footing.
The Cross-Encoder achieves the highest F1 (0.298) and precision (0.272),
followed by the MLP (0.250), LightGBM (0.190), GPT-4o (0.105), and the
retrieval-only baseline (0.136).

GPT-4o's recall collapses to 0.100 --- well below the retrieval-only
baseline (0.553) --- with high variance ($\text{SE}=0.004$, $\mu=0.105$),
indicating that correct returns are concentrated in a small fraction of value sets. Of all the codes returned by GPT-4o across all value sets, 48.6\% are absent from the full test corpus, confirming hallucination as the dominant failure mode for zero-shot generation over structured clinical
vocabularies.

\subsubsection{Stratification by value set type.}
Figure~\ref{fig:stratified}a shows F1 by value set type.
The Cross-Encoder leads in every category, with the largest absolute gains
over the MLP on Condition/Diagnosis (+0.078) and Lab/Observation (+0.060).
GPT-4o remains competitive on Condition/Diagnosis (F1~0.277) and Procedure
(0.236) --- categories where concepts are expressible in natural language and
code sets are compact --- but degrades to near zero on Medication (F1~0.003)
and Condition/Clinical (0.029), which require enumerating large numbers of
RxNorm and SNOMED-CT codes not reliably memorized from pretraining.

\begin{figure}[!h]
\centering
\begin{minipage}{\linewidth}
    \centering
    \includegraphics[width=0.7\linewidth]{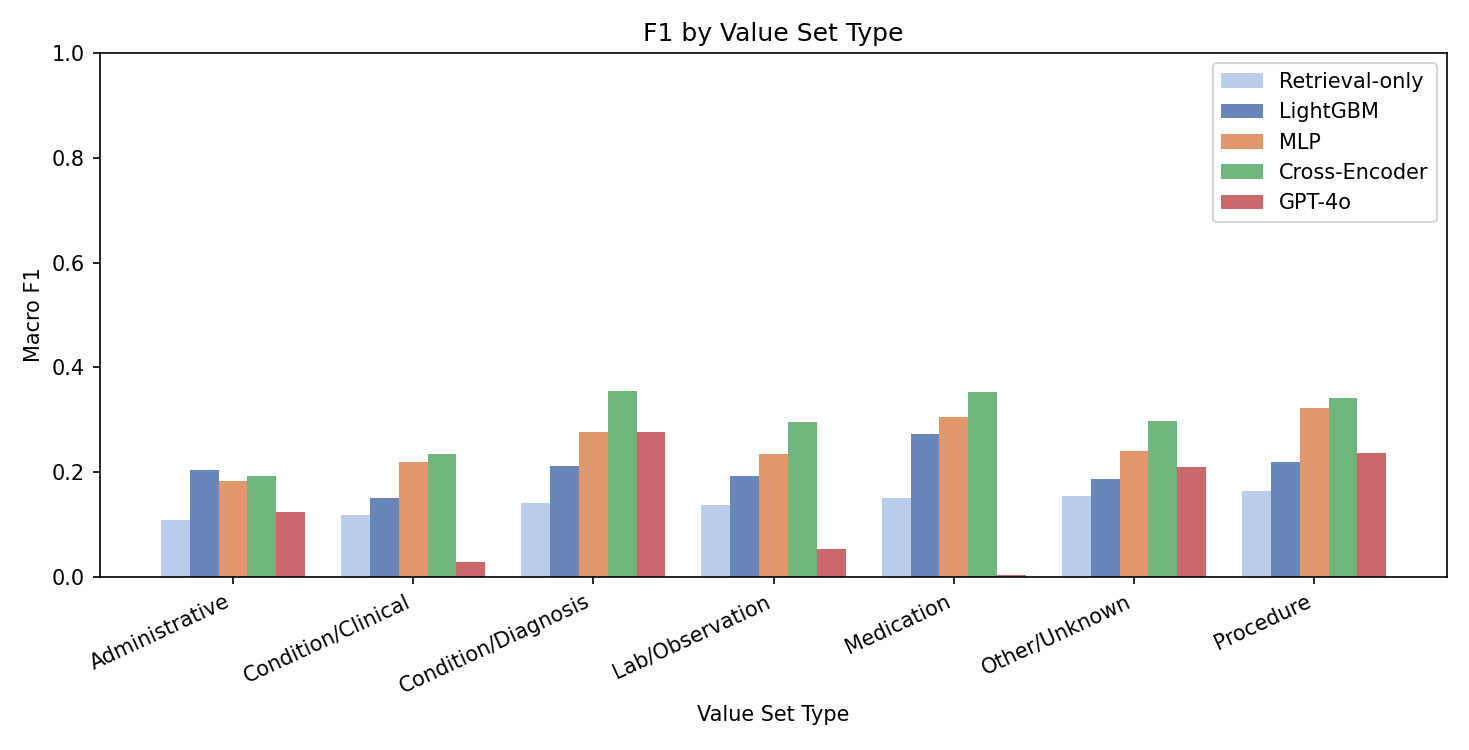}
    \label{fig:stratified_type}
\end{minipage}
\medskip
\begin{minipage}{\linewidth}
    \centering
    \includegraphics[width=0.7\linewidth]{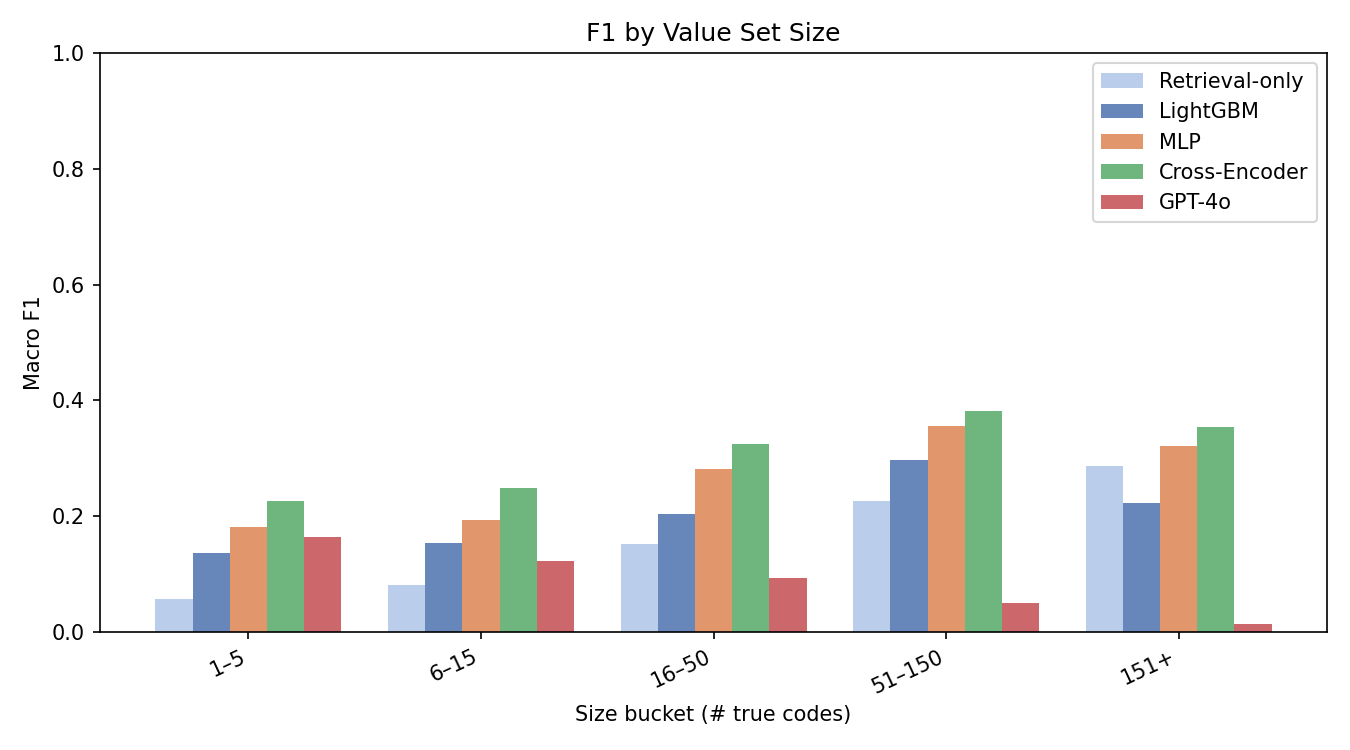}
    \label{fig:stratified_size}
\end{minipage}
\caption{Stratified value-set-level F1 across clinical type (a) and value set
size (b).}
\label{fig:stratified}
\end{figure}

\subsubsection{Stratification by value set size.}
Figure~\ref{fig:stratified}b reveals a crossover between GPT-4o and the
classifiers as a function of value set size.
For very small value sets (1--5 true codes), GPT-4o (F1~0.164) is
competitive with LightGBM (0.137) and approaches the MLP (0.182) and
Cross-Encoder (0.226).
Beyond 15 codes, GPT-4o degrades monotonically (F1~0.014 for sets with
$>$150 codes), while all three classifiers improve with set size; the
Cross-Encoder reaches F1~0.354 on the largest sets.
This shows that RASC's estimation advantage over direct generation
grows with target set size, as the pool-constrained classification problem
becomes increasingly tractable relative to unconstrained generation over
$|U| \approx 10^5$--$10^6$ identifiers.

\subsubsection{Precision at full retrieval coverage.}
On the 905 value sets where $\mathrm{RR}@K = 1.0$ (candidate pool contains
every true code), any precision deficit is attributable entirely to the model. Macro-averaged precision on this subset is: Cross-Encoder~0.438, MLP~0.410, LightGBM~0.349, GPT-4o~0.199, and retrieval-only~0.160.
The gap between the Cross-Encoder and the retrieval-only baseline quantifies
the classification value added under ideal retrieval, and suggests that
further gains are achievable through improved retrieval coverage of the
remaining value sets.

\subsection{Zero-Shot LLM vs.\ LLM-as-Classifier}
\label{sec:results:llm_pool}

The zero-shot GPT-4o results in Table~\ref{tab:results_vsl} leave open whether the classifier--LLM gap reflects a fundamental limitation of LLM-based generation or simply the absence of retrieval grounding. To isolate this, we provide GPT-4o with the full candidate pool in-context and
re-evaluate on a stratified subsample of 100 value sets with $|S^*| \le 50$
(to keep pool-sized prompts tractable in number of tokens).

\begin{table}[h]
\centering
\caption{GPT-4o zero-shot vs.\ pool-grounded, macro-averaged over 95 value sets
with $|S^*| \le 50$.}
\label{tab:llm_pool}
\begin{tabular}{lccc}
\toprule
Condition & Precision & Recall & F1 \\
\midrule
GPT-4o zero-shot & 0.131 & 0.086 & 0.089 \\
GPT-4o w/ pool   & 0.284 & 0.456 & 0.294 \\
\bottomrule
\end{tabular}
\label{tab:results_vsl}
\end{table}

Pool grounding triples F1 (Table~\ref{tab:llm_pool}), with the gain driven
almost entirely by recall. GPT-4o recognises
relevant codes when shown them but cannot generate them reliably from memory.
This suggests that LLM based generation can be tractable if the LLM is treated as a classifier grounded in a pool of retrieved codes. However, such an approach will likely not be as scalable or cost-effective as using a classifier model.

\section{Benchmark Release}
\label{sec:benchmark}

We provide code to reproduce the VSAC data used for model training and evaluation, to support reproducible research on corpus-grounded set completion. Since VSAC data is distributed under a UMLS license that precludes direct redistribution, we provide a download script that accepts a UMLS API key, bulk-fetches all value set expansions via the FHIR \texttt{ValueSet/\$expand} endpoint, and reconstructs the corpus used in this paper. A split manifest---containing OIDs, split assignments,
$\mathrm{RR}@K$ scores, value set type, and publisher for all 11,803 value
sets---is released as a lightweight index (no VSAC content) that users can
match against their local download to recover identical train/validation/test
splits. The candidate pool construction pipeline then produces the labeled \texttt{(value set, candidate code)} pairs from the downloaded corpus, with deterministic retrieval and de-duplication ensuring reproducible datasets across users. The embedding-model-agnostic design means any representation---dense embeddings, sparse retrievers, or hand-crafted features---can be evaluated against the same labels. The training code for all three classifiers is also made available in our github repository: \href{https://github.com/mukhes3/RASC}{https://github.com/mukhes3/RASC}.


\section{Conclusion}
\label{sec:discussion}

We introduced RASC, a two-stage framework for set completion over large discrete vocabularies
that replaces de novo generation with retrieval followed by binary classification. The key observation
is that expert-curated set corpora, present in most knowledge-intensive domains, can be leveraged
to reduce an intractable generation problem over $|\mathcal{U}|$ items to a tractable classification
problem over a much smaller candidate pool. Under standard margin and concentration assumptions,
this reduction yields a sample complexity gap of $O(\log K)$ versus $O(\log N)$, a condition
verified empirically: the classifier advantage over zero-shot generation grows monotonically with
value set size, directly consistent with the theoretical prediction.

On 11,803 VSAC value sets, RASC classifiers substantially outperform zero-shot GPT-4o generation, which returns codes absent from VSAC entirely at a 48.6\% rate and achieves value-set-level F1 of 0.105. Providing GPT-4o with the retrieved candidate pool in-context closes most of this gap on small value sets, suggesting that retrieval grounding might enable better quality LLM based generation. However, such an approach would suffer from higher cost and poor scalability. Cross-organization generalization to two entirely held-out publisher stewards further suggests that the learned representations reflect clinically meaningful structure rather than publisher-specific authoring patterns.

Retrieval quality remains the
binding constraint: on value sets with perfect retrieval coverage (RR@$K = 1.0$), the cross-encoder
achieves precision 0.438, suggesting that classification is effective when the pool is adequate and
that retrieval improvement is the most promising direction for further gains. The framework also
assumes a high-quality retrieval corpus exists; its behavior in domains where the corpus is sparse,
noisy, or rapidly evolving remains to be characterized. Future work will look at extending RASC
to other application domains such as gene panel construction and systematic review inclusion,
where the same structural conditions hold: large structured vocabularies, small target sets, and
existing curated collections that encode prior expert judgment. We release code to generate the identical VSAC training and evaluation data to enable the use of a consistent benchmark for future efforts in this direction.


\bibliographystyle{acm}
\bibliography{ref}

@misc{vsac2013,
  author       = {{National Library of Medicine}},
  title        = {Value Set Authority Center ({VSAC})},
  howpublished = {\url{https://vsac.nlm.nih.gov}},
  year         = {2013},
  note         = {Accessed 2025}
}

@article{steele2017quality,
  author    = {Steele, Nathan A and Bhattacharyya, Sanjukta
               and Bhagwat, Manasi and Morales, Allyn
               and Desai, Jay R and Kho, Abel N},
  title     = {Quality and consistency of {VSAC} value sets for
               clinical quality measurement},
  journal   = {Journal of the American Medical Informatics Association},
  volume    = {24},
  number    = {4},
  pages     = {716--722},
  year      = {2017},
  publisher = {Oxford University Press},
  doi       = {10.1093/jamia/ocw183}
}

@article{mo2015desiderata,
  author    = {Mo, Huan and Thompson, William K and Rasmussen, Luke V
               and Pacheco, Jennifer A and Jiang, Guoqian
               and Kiefer, Richard and Zhu, Qian and Xu, Jie
               and Montague, Elizabeth and Waitman, Lemuel R
               and others},
  title     = {Desiderata for computable representations of
               electronic health records-driven phenotype algorithms},
  journal   = {Journal of the American Medical Informatics Association},
  volume    = {22},
  number    = {6},
  pages     = {1220--1230},
  year      = {2015},
  publisher = {Oxford University Press},
  doi       = {10.1093/jamia/ocv112}
}

@inproceedings{liu2021sapbert,
  author    = {Liu, Fangyu and Shareghi, Ehsan and Meng, Zaiqiao
               and Basaldella, Marco and Collier, Nigel},
  title     = {Self-Alignment Pretraining for Biomedical Entity
               Representations},
  booktitle = {Proceedings of the 2021 Conference of the North
               American Chapter of the Association for
               Computational Linguistics: Human Language
               Technologies ({NAACL-HLT})},
  pages     = {4228--4238},
  year      = {2021},
  publisher = {Association for Computational Linguistics},
  doi       = {10.18653/v1/2021.naacl-main.334}
}

@inproceedings{mullenbach2018explainable,
  author    = {Mullenbach, James and Wiegreffe, Sarah
               and Duke, Jon and Sun, Jimeng and Eisenstein, Jacob},
  title     = {Explainable Prediction of Medical Codes from
               Clinical Text},
  booktitle = {Proceedings of the 2018 Conference of the North
               American Chapter of the Association for
               Computational Linguistics: Human Language
               Technologies ({NAACL-HLT})},
  pages     = {1101--1111},
  year      = {2018},
  publisher = {Association for Computational Linguistics},
  doi       = {10.18653/v1/N18-1100}
}

@inproceedings{huang2022plm,
  author    = {Huang, Chao-Wei and Tsai, Shang-Chi
               and Chen, Hen-Hsen},
  title     = {{PLM-ICD}: Automatic {ICD} Coding with Pretrained
               Language Models},
  booktitle = {Proceedings of the 4th Clinical Natural Language
               Processing Workshop},
  pages     = {10--20},
  year      = {2022},
  publisher = {Association for Computational Linguistics},
  doi       = {10.18653/v1/2022.clinicalnlp-1.2}
}

@article{kirby2016phekb,
  author    = {Kirby, Jacqueline C and Speltz, Peter
               and Rasmussen, Luke V and Basford, Melissa
               and Gottesman, Omri and Peissig, Peggy L
               and others},
  title     = {{PheKB}: a catalog and workflow for creating
               electronic phenotype algorithms for
               transportability},
  journal   = {Journal of the American Medical Informatics
               Association},
  volume    = {23},
  number    = {6},
  pages     = {1046--1052},
  year      = {2016},
  publisher = {Oxford University Press},
  doi       = {10.1093/jamia/ocv202}
}

@article{yu2018phenorm,
  author    = {Yu, Sheng and Chakrabortty, Abhishek
               and Ho, Yu-Han and Hidalgo, Bertrand
               and Denny, Joshua C and Blessing, Judith
               and others},
  title     = {{PheNorm}: A Gene-Regularized Topic Model for
               Unsupervised {EHR} Phenotyping},
  journal   = {Journal of the American Medical Informatics
               Association},
  volume    = {25},
  number    = {1},
  pages     = {54--60},
  year      = {2018},
  publisher = {Oxford University Press},
  doi       = {10.1093/jamia/ocx111}
}

@inproceedings{lewis2020retrieval,
  author    = {Lewis, Patrick and Perez, Ethan and Piktus, Aleksandra
               and Petroni, Fabio and Karpukhin, Vladimir
               and Goyal, Naman and K{\"u}ttler, Heinrich
               and Lewis, Mike and Yih, Wen-tau and Rockt{\"a}schel, Tim
               and others},
  title     = {Retrieval-Augmented Generation for
               Knowledge-Intensive {NLP} Tasks},
  booktitle = {Advances in Neural Information Processing Systems
               ({NeurIPS})},
  volume    = {33},
  pages     = {9459--9474},
  year      = {2020}
}

@article{sivarajkumar2024gpt4,
  author    = {Sivarajkumar, Sonish and Kelley, Mark
               and Samolyk-Mazzanti, Alyssa and Visweswaran, Shyam
               and Wang, Yanshan},
  title     = {An Empirical Evaluation of Prompting Strategies
               for Large Language Models in Zero-Shot Clinical
               Natural Language Processing},
  journal   = {Journal of the American Medical Informatics
               Association},
  volume    = {31},
  number    = {9},
  pages     = {1935--1945},
  year      = {2024},
  publisher = {Oxford University Press},
  doi       = {10.1093/jamia/ocae122}
}

@article{johnson2019faiss,
  author    = {Johnson, Jeff and Douze, Matthijs and J{\'e}gou, Herv{\'e}},
  title     = {Billion-scale similarity search with {GPUs}},
  journal   = {IEEE Transactions on Big Data},
  volume    = {7},
  number    = {3},
  pages     = {535--547},
  year      = {2021},
  publisher = {IEEE},
  doi       = {10.1109/TBDATA.2019.2921572}
}

@article{tsoumakas2007multi,
  author    = {Tsoumakas, Grigorios and Katakis, Ioannis},
  title     = {Multi-Label Classification: An Overview},
  journal   = {International Journal of Data Warehousing
               and Mining},
  volume    = {3},
  number    = {3},
  pages     = {1--13},
  year      = {2007},
  doi       = {10.4018/jdwm.2007070101}
}

@inproceedings{bhatia2015sparse,
  author    = {Bhatia, Kush and Jain, Himanshu and Kar, Purushottam
               and Varma, Manik and Jain, Prateek},
  title     = {Sparse Local Embeddings for Extreme Multi-label
               Classification},
  booktitle = {Advances in Neural Information Processing Systems
               ({NeurIPS})},
  volume    = {28},
  pages     = {730--738},
  year      = {2015}
}

@article{zhang2020deep,
  author    = {Zhang, Xingyou and Lipman, Yaron and Lee, Honglak},
  title     = {Deep Set Prediction Networks},
  booktitle = {Advances in Neural Information Processing Systems
               ({NeurIPS})},
  volume    = {32},
  year      = {2019}
}

@misc{panelapp,
  author       = {{Genomics England}},
  title        = {{PanelApp}: A publicly available gene panel
                  repository},
  howpublished = {\url{https://panelapp.genomicsengland.co.uk}},
  year         = {2019},
  note         = {Accessed 2025}
}

@article{o2019using,
  author    = {O'Mara-Eves, Alison and Thomas, James
               and McNaught, John and Miwa, Makoto
               and Ananiadou, Sophia},
  title     = {Using text mining for study identification in
               systematic reviews: a systematic review of
               current approaches},
  journal   = {Systematic Reviews},
  volume    = {4},
  number    = {1},
  pages     = {5},
  year      = {2015},
  publisher = {BioMed Central},
  doi       = {10.1186/2046-4053-4-5}
}

\newpage
\appendix

\section{VSAC Corpus Analysis}
\label{sec:eda}

We present an exploratory analysis of the entire \textbf{unfiltered} VSAC corpus to characterise the data distribution and motivate several design decisions in the main paper. 

\subsection{Temporal Growth}

Figure~\ref{fig:temporal} shows value set creation and update activity by
year. Growth has been substantial and accelerating: fewer than 100 value
sets were added in each of 2012 and 2013, rising to approximately 1,750 per
year in 2023--2024. More value sets were created in those two years alone
than in the entire period 2012--2020 combined, reflecting both the expansion
of eCQM reporting requirements and the growing adoption of FHIR-based quality
measurement. 

\begin{figure}[h]
    \centering
    \includegraphics[width=\textwidth]{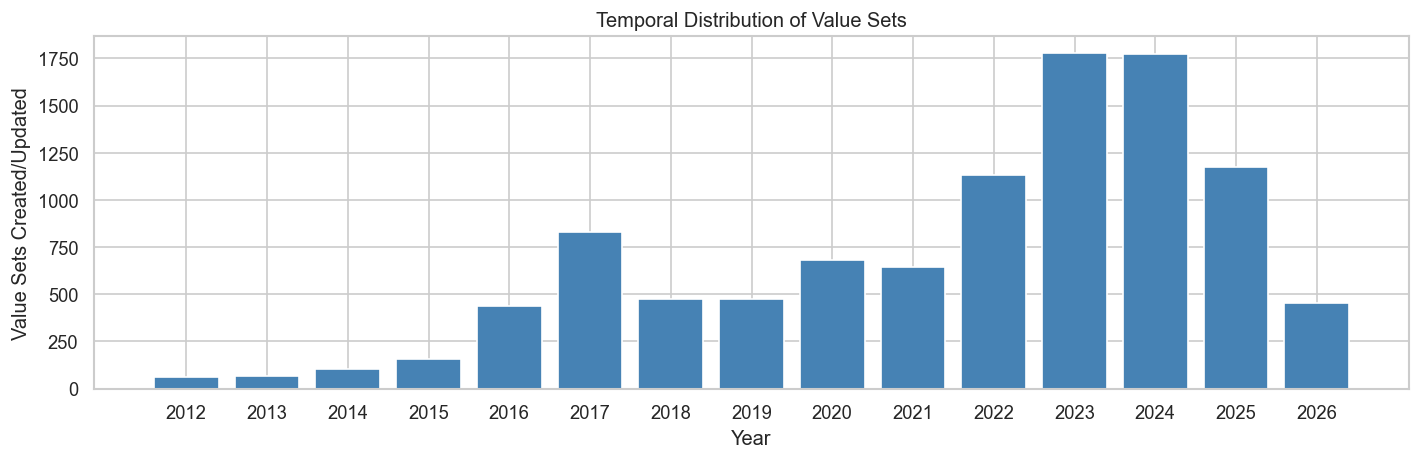}
    \caption{Temporal distribution of VSAC value sets by year of creation or
    last update. Growth accelerates sharply after 2021, reaching approximately
    1{,}750 value sets per year in 2023--2024.}
    \label{fig:temporal}
\end{figure}

\subsection{Terminology System and Value Set Type}

Figure~\ref{fig:codesystem} shows the distribution of value sets by primary
code system. SNOMED-CT dominates with approximately 4{,}250 value sets,
followed by ICD-10-CM (${\sim}2{,}500$), LOINC (${\sim}1{,}500$), and RxNorm
(${\sim}1{,}300$). Critically, the right panel shows that approximately 8{,}800 of the ${\sim}10{,}700$ analysed value sets draw codes from a single
terminology system, with roughly 1{,}000 spanning two systems and approximately 350 spanning three. 

\begin{figure}[h]
    \centering
    \includegraphics[width=\textwidth]{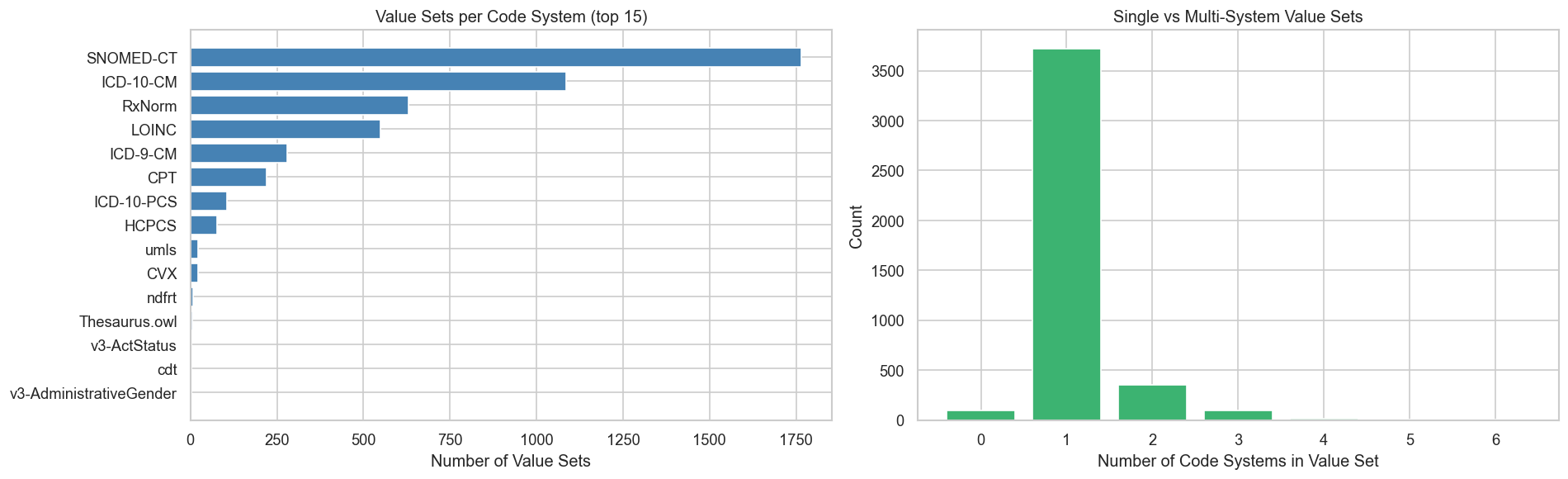}
    \caption{Left: number of value sets per code system (top 15). SNOMED-CT
    and ICD-10-CM together account for over 60\% of the corpus. Right:
    distribution of the number of distinct code systems per value set; over
    80\% of value sets are single-system.}
    \label{fig:codesystem}
\end{figure}

Figure~\ref{fig:type} shows the inferred semantic type distribution alongside
value set size stratified by type. Condition/Clinical and
Condition/Diagnosis together account for over 60\% of the
corpus. Medication value sets exhibit the largest size
variance, with a heavy tail exceeding 200 codes.
Lab/Observation value sets  are comparatively compact. This size heterogeneity motivates stratified evaluation across types.

\begin{figure}[h]
    \centering
    \includegraphics[width=\textwidth]{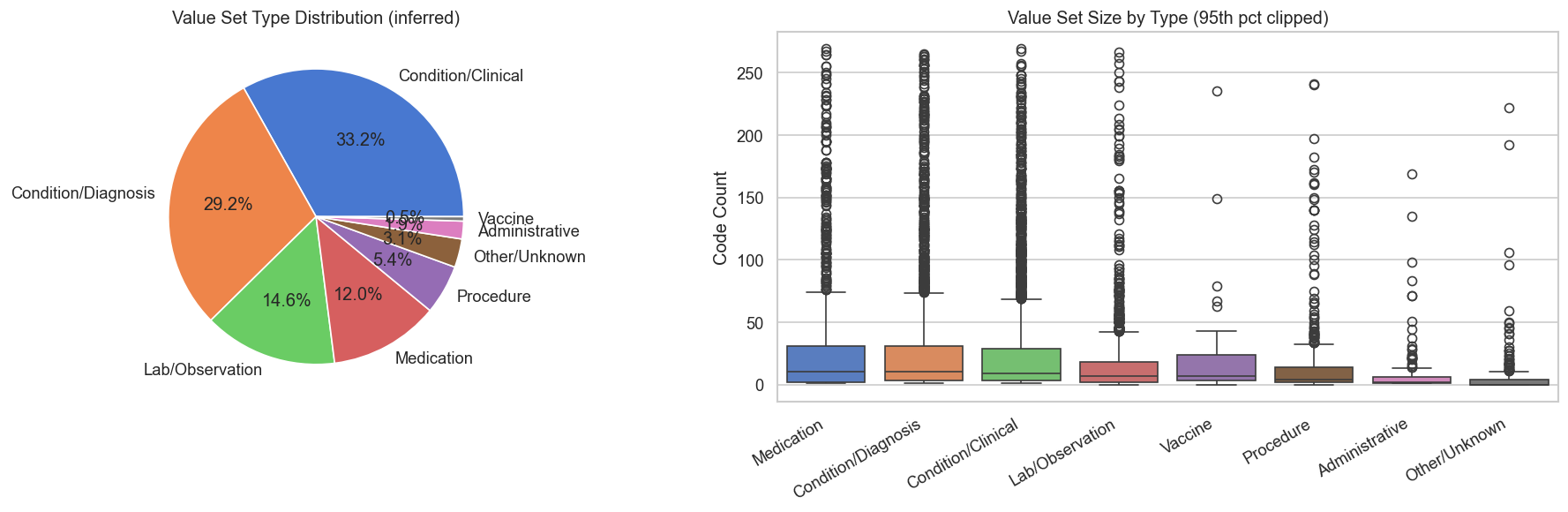}
    \caption{Left: inferred value set type distribution. Condition types
    dominate the corpus. Right: value set size (code count) by type, clipped
    at the 95th percentile. Medication value sets exhibit substantially higher
    size variance than other types.}
    \label{fig:type}
\end{figure}

\subsection{Description Coverage}

Figure~\ref{fig:desc} shows the fraction of value sets carrying a
human-authored textual description, and the distribution of description
lengths among those that do. Only 19.6\% of value sets include a description
field; the remaining 80.4\% provide a title alone. Among value sets with
descriptions, lengths are right-skewed, peaking between 5 and 15 words,
with most consisting of brief definitional phrases rather than detailed
clinical narrative. This near-absence of descriptions has a direct
methodological consequence: augmenting title embeddings with description
text would produce an inconsistent representation across the corpus, since
no description signal exists for four out of five value sets at inference
time. We therefore embed titles only.

\begin{figure}[h]
    \centering
    \includegraphics[width=\textwidth]{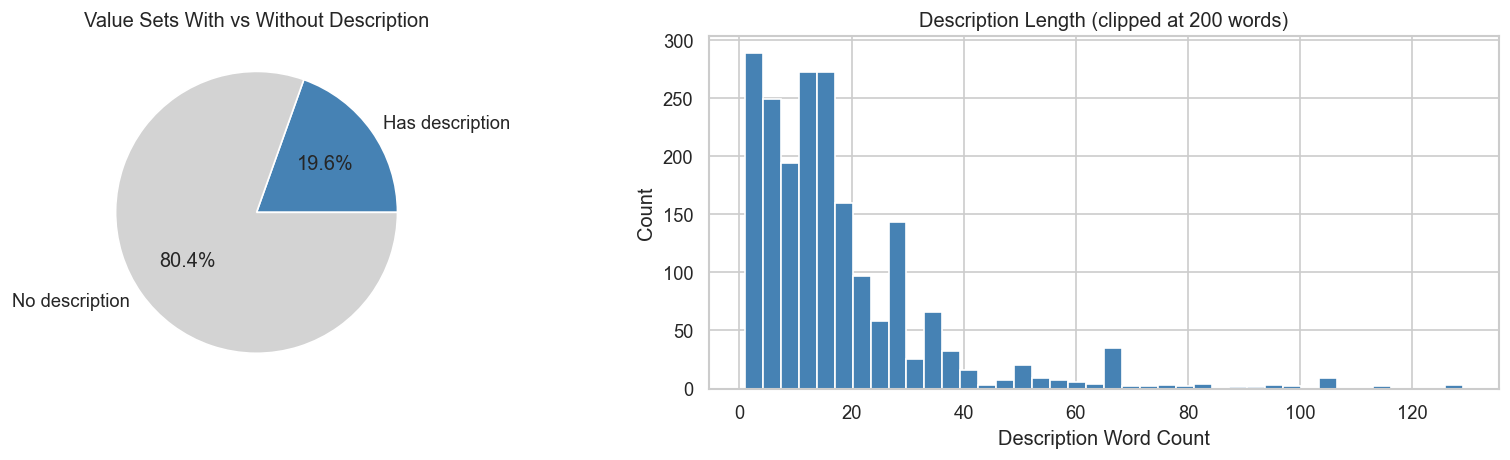}
    \caption{Left: fraction of value sets with a human-authored textual
    description. Only 19.6\% carry one. Right: distribution of description
    word count among value sets that have a description; most are brief
    (under 20 words).}
    \label{fig:desc}
\end{figure}

\subsection{Publisher Distribution}

Figure~\ref{fig:publisher} shows the top 20 publisher organizations by value
set count. The distribution is highly concentrated: CSTE Steward alone
contributes approximately over 1500 value sets, followed by NCQA PHEMUR
, Clinical Architecture , and HL7 Patient Care WG
, with a long tail of hundreds of smaller organizations across
more than 800 distinct publishers in total.

\begin{figure}[h]
    \centering
    \includegraphics[width=0.85\textwidth]{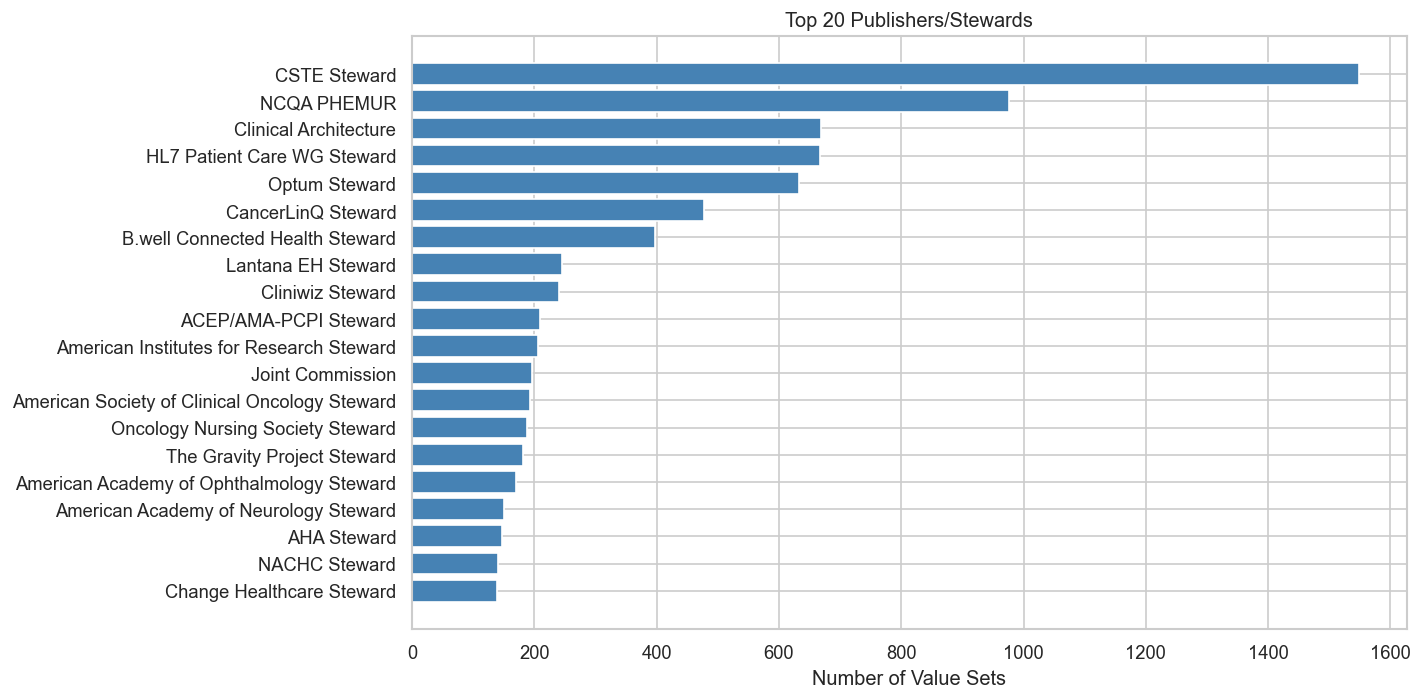}
    \caption{Top 20 VSAC publisher organizations by number of value sets.
    The distribution is highly concentrated; the top 5 publishers account
    for a disproportionate share of the corpus.}
    \label{fig:publisher}
\end{figure}

\section{GPT-4o Prompt}
\label{sec:appendix:prompt}

Each value set is evaluated with a single zero-shot call.
The system prompt and user message are shown below.

\paragraph{System prompt.}
\begin{quote}
\ttfamily\small
You are a clinical informaticist expert in medical terminologies including
SNOMED CT, ICD-10-CM, RxNorm, LOINC, ICD-10-PCS, CPT, HCPCS, and NCI
Thesaurus.

Your task is to generate the expansion of a clinical value set given its
name and the allowed code systems. A value set expansion is a list of
standardised clinical codes that define a specific clinical concept.

Return ONLY a JSON array. Each element must be an object with exactly
these fields: ``code'' (the exact code identifier), ``system'' (one of
the allowed systems listed for this value set), and ``display'' (the
human-readable name for the code).

Only use codes from the allowed systems listed for this value set. Use
real, valid codes --- do not invent codes. Be comprehensive but precise.
Return ONLY the JSON array, no explanation or markdown fences.
\end{quote}

\paragraph{User message.}
\begin{quote}
\ttfamily\small
Value set name: \textit{\{title\}}\\
Allowed code systems:\\
\phantom{xx}- \textit{\{system\textsubscript{1}\}}\\
\phantom{xx}- \textit{\{system\textsubscript{2}\}}\\
$\vdots$

Return the value set expansion as a JSON array of
\{"code", "system", "display"\} objects.
\end{quote}

\noindent
Temperature was set to~0 for deterministic outputs and \texttt{max\_tokens}
to~4{,}096. Allowed code systems were inferred from the candidate pool
rather than specified a priori, ensuring the prompt reflects only
information available at inference time. Responses were parsed as JSON
arrays; bare single-object responses were wrapped into a list, and system
URI strings were normalized to canonical short forms prior to matching
(e.g.\ \texttt{http://www.nlm.nih.gov/research/umls/rxnorm}
$\to$ \texttt{RxNorm}).

\section{Proof of the Sample Complexity Gap}
\label{sec:appendix_theory}

We provide the formal setup and proof for Corollary~\ref{cor:sample_complexity}.
The analysis compares direct open-world prediction over the full universe
$\mathcal{U}$ against retrieval-restricted prediction over a candidate pool
$\mathcal{C}(q) \subseteq \mathcal{U}$.

\begin{assumption}[Sparse labels]
\label{ass:sparse}
For every query $q$, $|Y(q)| \le s$ with $s \ll N$.
\end{assumption}

\begin{assumption}[Margin]
\label{ass:margin}
There exists $\gamma > 0$ such that for every query $q$,
$\theta_q(c) \ge \tau + \gamma$ for all $c \in Y(q)$, and
$\theta_q(c) \le \tau - \gamma$ for all $c \notin Y(q)$.
\end{assumption}

\begin{assumption}[Retrieval coverage and candidate size]
\label{ass:retrieval}
For each query $q$, the retriever outputs $\mathcal{C}(q)$ with
$|\mathcal{C}(q)| \le K \ll N$ and
$\Pr(Y(q) \subseteq \mathcal{C}(q)) \ge 1 - \varepsilon_{\mathrm{ret}}$.
\end{assumption}

\begin{assumption}[Uniform score estimation]
\label{ass:uniform}
There exists $\sigma > 0$ such that for every query $q$, every finite
$A \subseteq \mathcal{U}$, and every $t > 0$,
\[
    \Pr\!\left(
        \sup_{c \in A} |\hat\theta_n(q,c) - \theta_q(c)| > t
        \;\Big|\; q
    \right)
    \le
    2|A| \exp\!\left( -\frac{nt^2}{2\sigma^2} \right).
\]
This is a standard finite-class union-bound concentration: if estimation
errors are sub-Gaussian with variance $\sigma^2/n$ for each code $c$
individually, a union bound over $A$ yields the above, with the cost of
uniformity growing only as $\log|A|$.
\end{assumption}

For any $A \subseteq \mathcal{U}$, define the thresholded predictor
$\hat{Y}_A(q) := \{c \in A : \hat\theta_n(q,c) \ge \tau\}$.
The direct and retrieval-restricted predictors are then
$\hat{Y}_{\mathrm{dir}}(q) := \hat{Y}_{\mathcal{U}}(q)$ and
$\hat{Y}_{\mathrm{RASC}}(q) := \hat{Y}_{\mathcal{C}(q)}(q)$ respectively.

\begin{theorem}[Exact recovery probability]
\label{thm:main}
Under Assumptions~\ref{ass:sparse}--\ref{ass:uniform}, for every query $q$,
\[
    \Pr\!\bigl(\hat{Y}_{\mathrm{dir}}(q) \neq Y(q) \mid q \bigr)
    \;\le\;
    2N \exp\!\left( -\frac{n\gamma^2}{2\sigma^2} \right),
\]
and
\[
    \Pr\!\bigl(\hat{Y}_{\mathrm{RASC}}(q) \neq Y(q) \mid q \bigr)
    \;\le\;
    \varepsilon_{\mathrm{ret}}(q)
    \;+\;
    2K \exp\!\left( -\frac{n\gamma^2}{2\sigma^2} \right),
\]
where $\varepsilon_{\mathrm{ret}}(q) := \Pr(Y(q) \nsubseteq \mathcal{C}(q) \mid q)$.
Averaging over $q$,
\[
    \Pr\!\bigl(\hat{Y}_{\mathrm{RASC}}(q) \neq Y(q)\bigr)
    \;\le\;
    \varepsilon_{\mathrm{ret}}
    \;+\;
    2K \exp\!\left( -\frac{n\gamma^2}{2\sigma^2} \right).
\]
\end{theorem}

\begin{proof}
We first analyze $\hat{Y}_A(q)$ for a generic finite set $A \subseteq \mathcal{U}$.
If $\sup_{c \in A} |\hat\theta_n(q,c) - \theta_q(c)| < \gamma$, then for
every $c \in Y(q) \cap A$ we have $\hat\theta_n(q,c) \ge \theta_q(c) - \gamma \ge \tau$,
and for every $c \in A \setminus Y(q)$ we have
$\hat\theta_n(q,c) \le \theta_q(c) + \gamma \le \tau$.
Hence $\hat{Y}_A(q) = Y(q) \cap A$ on this event, so
\[
    \Pr\!\bigl(\hat{Y}_A(q) \neq Y(q) \cap A \mid q\bigr)
    \;\le\;
    \Pr\!\left(
        \sup_{c \in A} |\hat\theta_n(q,c) - \theta_q(c)| \ge \gamma
        \mid q
    \right)
    \;\le\;
    2|A| \exp\!\left( -\frac{n\gamma^2}{2\sigma^2} \right),
\]
where the last inequality applies Assumption~\ref{ass:uniform} with $t = \gamma$.

Taking $A = \mathcal{U}$ and using $Y(q) \cap \mathcal{U} = Y(q)$ gives the
bound for $\hat{Y}_{\mathrm{dir}}(q)$.

For $\hat{Y}_{\mathrm{RASC}}(q)$, exact recovery additionally requires the
event $E_q := \{Y(q) \subseteq \mathcal{C}(q)\}$.
By the law of total probability,
\[
    \Pr\!\bigl(\hat{Y}_{\mathrm{RASC}}(q) \neq Y(q) \mid q\bigr)
    \;\le\;
    \Pr(E_q^c \mid q)
    \;+\;
    \Pr\!\bigl(\hat{Y}_{\mathcal{C}(q)}(q) \neq Y(q) \cap \mathcal{C}(q)
        \mid E_q, q\bigr).
\]
By Assumption~\ref{ass:retrieval}, $\Pr(E_q^c \mid q) = \varepsilon_{\mathrm{ret}}(q)$.
On $E_q$ we have $Y(q) \cap \mathcal{C}(q) = Y(q)$, and since
$|\mathcal{C}(q)| \le K$, the generic bound above gives
\[
    \Pr\!\bigl(\hat{Y}_{\mathcal{C}(q)}(q) \neq Y(q) \cap \mathcal{C}(q)
        \mid E_q, q\bigr)
    \;\le\;
    2K \exp\!\left( -\frac{n\gamma^2}{2\sigma^2} \right).
\]
Combining and averaging over $q$ completes the proof.
\end{proof}

Theorem~\ref{thm:main} separates the two error sources: an approximation term
$\varepsilon_{\mathrm{ret}}$ from retrieval misses, and an estimation term
controlled by $K$ rather than $N$.
Corollary~\ref{cor:sample_complexity} follows immediately by solving each
bound for $n$.

\begin{proof}[Proof of Corollary~\ref{cor:sample_complexity}]
Setting $2N\exp(-n\gamma^2/2\sigma^2) \le \delta$ and solving for $n$ gives
the direct prediction condition.
For RASC, on the event $E_q$ the estimation term satisfies
$2K\exp(-n\gamma^2/2\sigma^2) \le \delta - \varepsilon_{\mathrm{ret}}$;
solving for $n$ gives the stated condition.
\end{proof}

\end{document}